\documentclass[twoside]{article}

\usepackage[accepted]{aistats2024}

\RequirePackage{natbib}

\setcitestyle{authoryear,round,citesep={;},aysep={,},yysep={;}}

\usepackage{xcolor}
\definecolor{mydarkblue}{RGB}{0, 0, 139}

\usepackage[colorlinks=true,
    linkcolor=mydarkblue,
    citecolor=mydarkblue,
    filecolor=mydarkblue,
    urlcolor=mydarkblue]{hyperref}

\usepackage{url}

\usepackage{amsmath}
\usepackage{amsthm}
\usepackage{amsbsy}
\usepackage{amssymb} 

\usepackage{nicefrac}
\usepackage{mathtools}

\usepackage{tikz}
\usepackage[framemethod=TikZ]{mdframed}
\usepackage{thmtools}

\definecolor{shadecolor}{rgb}{0.92, 0.94, 0.96}
\declaretheoremstyle[
headfont=\normalfont\bfseries,
notefont=\mdseries, notebraces={(}{)},
bodyfont=\normalfont,
postheadspace=0.5em,
spaceabove=6pt,
mdframed={
  skipabove=8pt,
  skipbelow=8pt,
  hidealllines=true,
  backgroundcolor={shadecolor},
  innerleftmargin=4pt,
  innerrightmargin=4pt,
  roundcorner=3pt}
]{shaded}

\declaretheorem[style=shaded,within=section]{definition}
\declaretheorem[style=shaded,sibling=definition]{theorem}
\declaretheorem[style=shaded,sibling=definition]{proposition}

\declaretheorem[style=shaded,sibling=definition]{corollary}

\declaretheorem[style=shaded,sibling=definition]{lemma}

\declaretheorem[style=shaded,sibling=definition]{example}

\usepackage{xcolor}
\usepackage{color}

\usepackage{hyperref}
\usepackage{textcomp}
\usepackage{cases}  

\newcommand{\R}{\mathbb{R}} 
\newcommand{\N}{\mathbb{N}} 

\usepackage[colorinlistoftodos,bordercolor=orange,backgroundcolor=orange!20,linecolor=orange]{todonotes}

\newcommand{\defeq}{\coloneqq} 

\newcommand{\st}{\;|\;} 
\newcommand{\norm}[1]{ \| #1 \|}      

\DeclareMathOperator{\dom}{dom}         
\DeclareMathOperator{\argminin}{argmin}  
  
\newcommand{\argmin}[1]{\underset{#1}{\argminin}}

\DeclareMathOperator{\prox}{prox}       

\DeclareMathOperator{\conv}{conv}       

\usepackage{lineno}
\usepackage[utf8]{inputenc} 
\usepackage[T1]{fontenc}    
\usepackage{url}            
\usepackage{booktabs}       

\usepackage{enumitem}
\usepackage{appendix}
\usepackage{wrapfig}
\usepackage{xspace}
\usepackage{enumitem}
\usepackage{tabularx}

\usepackage{algpseudocode}
\usepackage{algorithm}

\usepackage[nameinlink, capitalize]{cleveref}

\usepackage{tikz}
\usepackage{standalone}
\usetikzlibrary{shapes,arrows,fit,positioning,calc}
\usepackage{calc}
\usepackage{longtable}

\usepackage[most]{tcolorbox}
\usepackage{graphicx,subfigure}
\usepackage{caption} 
\newcommand{\avg}{\operatorname{avg}}

\newcommand{\eg}[0]{\emph{e.g.},~}
\newcommand{\ie}[0]{\emph{i.e.},~}

\newcommand{\mat}[1]{\ensuremath{\MakeUppercase{\mathbf{#1}}}}

\newcommand{\abs}[1]{\left| #1 \right|}
\newcommand{\floor}[1]{\lfloor #1 \rfloor}

\renewcommand{\R}{\mathcal{R}}

\renewcommand{\Re}{\mathbb{R}}

\newcommand{\pisort}{\Pi_\mathrm{sort}}

\newcommand{\dfn}[1]{\textit{#1}}
\newcommand{\alg}[1]{\textsc{#1}}

\newcommand{\ica}{\alg{ImminentCollisions}\xspace}
\newcommand{\eca}{\alg{EndCollisions}\xspace}
\newcommand{\sca}{\alg{SearchCollisions}\xspace}

\newcommand{\indic}[1]{\ensuremath{\mathbf{1}_{[#1]}}}

\newlength{\DepthReference}
\newlength{\HeightReference}
\newlength{\Width}%

\newcommand{\MyColorBox}[2]%
{%
    \settowidth{\Width}{#2}%
    \colorbox{#1}%
    {%
        \raisebox{-\DepthReference}%
        {%
                \parbox[b][\HeightReference+\DepthReference][c]{\Width}{\centering#2}%
        }%
    }%
}

\definecolor{grayishyellow}{rgb}{0.96, 0.95, 0.88}
\newcommand{\hl}[1]{#1}

\usepackage{xpatch}

\begin{document}

\algblock{ParDo}{End}
\algblock{ParFor}{EndFor}
\algblock{ParForAll}{EndFor}

\algnewcommand\algorithmicparfor{\textbf{for}}
\algnewcommand\algorithmicparforall{\textbf{for all}}
\algnewcommand\algorithmicpardo{\textbf{pardo}}
\algrenewtext{ParDo}{\algorithmicpardo}
\algrenewtext{End}{\algorithmicend}
\algrenewtext{ParFor}[1]{\algorithmicparfor\ #1\ \algorithmicpardo}
\algrenewtext{ParForAll}[1]{\algorithmicparforall\ #1\ \algorithmicpardo}

\algrenewcommand\algorithmicindent{0.8em}%

\runningauthor{Mehran Shakerinava$^\ast$, Motahareh Sohrabi$^\ast$, Siamak Ravanbakhsh, Simon Lacoste-Julien}

\twocolumn[

\aistatstitle{Weight-Sharing Regularization}

\aistatsauthor{ Mehran Shakerinava$^{*,1,2}$ \And Motahareh Sohrabi$^{*,2,3}$ }\aistatsauthor{  Siamak Ravanbakhsh$^{1,2,4}$ \And Simon Lacoste-Julien$^{2,3,4}$ }

\aistatsaddress{ $^1$McGill University \And  $^2$Mila - Quebec AI Institute \And $^3$Université de Montréal \And $^4$CIFAR AI Chair} ]

\begin{abstract}
    Weight-sharing is ubiquitous in deep learning. Motivated by this, we propose a ``weight-sharing regularization'' penalty on the weights $w \in \Re^d$ of a neural network, defined as $\mathcal{R}(w) = \frac{1}{d - 1}\sum_{i > j}^d |w_i - w_j|$. We study the proximal mapping of $\mathcal{R}$ and provide an intuitive interpretation of it in terms of a physical system of interacting particles. We also parallelize existing algorithms for $\prox_\R$ (to run on GPU) and find that one of them is fast in practice but slow ($O(d)$) for worst-case inputs. Using the physical interpretation, we design a novel parallel algorithm which runs in $O(\log^3 d)$ when sufficient processors are available, thus guaranteeing fast training. Our experiments reveal that weight-sharing regularization enables fully connected networks to learn convolution-like filters even when pixels have been shuffled while convolutional neural networks fail in this setting. Our code is available on \href{https://github.com/motahareh-sohrabi/weight-sharing-regularization}{github}.
\end{abstract}

\section{INTRODUCTION}
All modern deep learning architectures, from Convolutional Neural Networks (CNNs)~\citep{lecun1989backpropagation} to transformers~\citep{vaswani2017attention}, use some form of \textit{weight-sharing}, \eg CNNs apply the same weights at different locations of the input image. More generally, requiring a linear function to be symmetric w.r.t. a permutation group induces weight-sharing in its corresponding matrix form~\citep{shawe1989building,ravanbakhsh2017equivariance}, \eg symmetry w.r.t. the group of cyclic permutations $C_n$ produces a circulant matrix, resulting in circular convolution.

\begin{figure}[t]
    \centering    \includegraphics[width=0.45\linewidth]{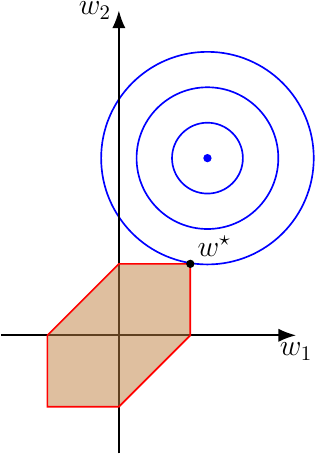}
    \caption{Depiction of the unregularized error function's contour lines (in blue), together with the constraint area for $\R + \ell_1$, where the optimal parameter vector $w$ is marked by $w^\star$. Notice that the weights are set to be equal (shared) in the solution.}
    \label{fig:regularization}
\end{figure}

fully connected networks with no structure or weight-sharing are prone to overfitting the dataset. A common technique used in machine learning to avoid overfitting and improve generalization is regularization~\citep{hastie2001elements}. Motivated by these facts, we propose \textit{weight-sharing regularization}. For a weight vector $w \in \Re^d$, the weight-sharing regularization penalty is defined as
\begin{equation}\label{eq:reg}
    \R(w) = \frac{1}{d - 1}\sum_{i > j}^d \abs{w_i - w_j}.
\end{equation}

Our goal is to train deep neural networks regularized with $\R$ (See \cref{fig:regularization}). By analogy to $\ell_1$ regularization where using a stochastic subgradient method (SGD) performs poorly for compression as it does not yield sufficiently small weights~\citep{spred}, we propose to use instead (stochastic) proximal gradient methods~\citep{atchade2017perturbed} which yield exactly tied weights. The efficient computation of the proximal update will be a central contribution of this paper.

Given a task-specific loss function $\mathcal{L}(w)$ we may construct a weight-sharing and $\ell_1$ regularized loss with coefficients $\alpha$ and $\beta$ as
\begin{equation}
    l(w) = \mathcal{L}(w) + \alpha \R(w) + \beta \ell_1(w).
\end{equation}
The proximal gradient descent update for $l$ with step-size $\eta$ is then given by
\begin{equation}
    w^{(t+1)} = \prox_{\eta (\alpha \R + \beta \ell_1)} \left( w^{(t)} - \eta \nabla \mathcal{L}(w^{(t)}) \right),
\end{equation}
where the proximal mapping is defined as
\begin{equation}
    \prox_f(x) = \argmin{u} \left( f(u)+\frac{1}{2}\norm{u-x}_2^2 \right).
\end{equation}

It follows from \citet[Corollary 4]{yu2013decomposing} that
\begin{equation}
    \prox_{\alpha\R + \beta \ell_1}(x) = \prox_{\beta \ell_1}\left(\prox_{\alpha\R}(x)\right).\footnote{See \cref{sec:prox_R_l1} for an alternative proof.}
\end{equation}

Therefore, the main problem lies in computing $\prox_\R$.

\subsection{Contributions}

In this work, we make the following contributions:

\begin{itemize}
    \item We provide a new formulation of the proximal mapping of $\R$ in terms of the solution to an Ordinary Differential Equation (ODE). We give an intuitive interpretation of this ODE in terms of a physical system of interacting particles. This interpretation helps provide an intuitive understanding of algorithms for $\prox_\R$.

    \item We analyze parallel versions of existing algorithms for $\prox_\R$ and show that, in the worst case, they provide no speedup compared to the best sequential algorithm. Based on our physical system analogy, we propose a novel parallel algorithm with a depth of $O(\log^3 d)$ that provides an exponential speedup for worst-case inputs.    

    \item In experiments, for the first time we are able to effectively apply weight-sharing regularization at scale in the context of deep neural networks and recover convolution-like filters in Multi-Layer Perceptrons (MLPs), even when pixels have been shuffled.
\end{itemize}

\subsection{Related Works}\label{sec:related_works}

\textbf{Clustered lasso.} Regularization with $\R$ and $\ell_1$ has been studied in a linear regression setting and is known as \textit{clustered lasso}~\citep{She2010sparse}. Clustered lasso is an instance of \textit{generalized lasso}~\citep{generalized_lasso}, where the regularization term is specified via a matrix $\mat{D}$ as $\norm{\mat{D} w}_1$. In clustered lasso, $\mat{D}$ is a tall matrix with $d + {d \choose 2}$ rows, where $d$ rows form a diagonal matrix for $\ell_1$, and each of the rest encodes the subtraction of a unique pair of weights for $\R$.

The proximal mapping of generalized lasso takes a simple form in the dual space and algorithms have been developed to exploit this fact~\citep{arnold2016efficient}. However, the dimensionality of the dual space is the same as the number of rows in $\mat{D}$, which makes dual methods infeasible for clustered lasso when $d$ is large. We avoid this problem in this work by staying in the primal space.

\textbf{Isotonic regression.} \citet{lin2019efficient} show that the proximal mapping of $\R$ can be obtained via the isotonic regression problem
\begin{equation}\label{eq:iso_regr}
    \min_{x \in \Re^d} \norm{y - x}_2^2 \quad \text{subject to} \quad x_1 \leq x_2 \leq ... \leq x_n.
\end{equation}
One can therefore use algorithms designed for isotonic regression~\citep{best1990active} to compute the proximal mapping of clustered lasso in $O(d \log d)$ time. Moreover, based on this connection, our proposed algorithm can also be used as a fast parallel solver for isotonic regression. In this work, however, there will be no need to understand isotonic regression. Our physical interpretation provides a complete alternative specification of the problem that is much more intuitive.

\citet{kearsley1996approach} propose a parallel algorithm for isotonic regression which we show to be \textit{incorrect} with a simple counterexample, included in \cref{app:counterexample}. 

\textbf{Symmetry.} Our work is also related to a body of work in invariant and equivariant deep learning with finite symmetry groups, as well as prior work on regularization techniques and network compression. 
In particular, using (soft) weight-sharing for regularization is motivated from two perspectives: symmetry, and minimum description length.

As shown in~\cite{shawe1989building,ravanbakhsh2017equivariance}, requiring a linear function to respect a permutation symmetry induces weight-sharing in its corresponding matrix form. The same idea can be expressed using a generalization of convolution to symmetry groups~\citep{cohen2016group}. Identification of weight-sharing patterns can be used for discovering symmetries of the data or task. \citet{zhou2020meta,yeh2022equivariance} pursue this goal in a meta-learning setting and using bilevel optimization, respectively. 

\textbf{Compression.} However, not all weight-sharing patterns correspond to a symmetry transformation.\footnote{In technical terms, a weight-sharing pattern corresponds to equivariance with respect to a permutation group, if and only if the parameters corresponding to the same orbit due to the action of the group on the rows and columns of the matrix are tied together.} 
Minimum Description Length (MDL) provides a more general motivation. According to MDL, the best model is the one that has the shortest description of the model itself plus its prediction errors~\citep{rissanen2007information}. Assuming a prior over the weights so as to achieve the best compression, MDL leads to regularization. 
In particular, the soft weight-sharing of~\citet{hinton1993keeping} is derived from this perspective and motivates follow-up works on network compression~\citep{ullrich2017soft}. While these works assume a Gaussian mixture prior for the weights, leading to $\ell_2$ regularization, we use sparsity inducing $\ell_1$ regularization to encourage exact weight-sharing. 
Indeed, assuming $\ell_0$ in \cref{eq:reg}, the objective is similar to vector quantization for compression; instead of using  $d c$ bits, where $c$ is the number of bits-per-weight, using exact weight-sharing with $k$ distinct weights, one requires only $kc + d \log k$ bits, leading to significant compression for small $k$.
This observation motivates the regularization of \cref{eq:reg} from the MDL point of view.

\subsection{Outline}

The required background on optimization and parallel computation for this paper is covered in \cref{sec:background}. We begin in \cref{sec:subdiff_R} by studying the subdifferential of $\R$. Next, in \cref{sec:prox_R}, we show that the proximal mapping can be obtained by simulating a physical system of interacting particles. In \cref{sec:algorithms}, we study and analyze parallel algorithms for computing the proximal mapping of $\R$. We introduce a novel low-depth parallel algorithm while showing that previously existing algorithms have high depth. In \cref{sec:rewinding}, we describe a method for controlling the undesirable bias of $\ell_1$-relaxation. In \cref{sec:experiments}, we apply our algorithm to training weight-sharing regularized deep neural networks with proximal gradient descent on GPU. We discuss limitations and directions for future work in \cref{sec:future_work} and conclude in \cref{sec:conclusion}.

\section{SUBDIFFERENTIAL OF $\R$}\label{sec:subdiff_R}

Obtaining the subdifferential of $\R$ will be essential to deriving its proximal mapping. We will do this by writing $\R$ as a point-wise maximum of linear functions. We will refer to each element of the weight vector $w$ as a weight.

The function \textit{$\R$ is a convex and piece-wise linear function}. It will be useful to keep in mind that an absolute value can be written as a $\max$, \eg $|x| = \max(x, -x)$.
Each piece of $\R$, when extended, forms a tangent hyperplane. Depending on the sign that we choose for each absolute value term of $\R$, we get the equation of one of these hyperplanes. The correct signs are determined by the ordering of weights. More specifically, there is a one-to-one correspondence between hyperplanes and the ordering of the weights. It follows that $\R$ has $d!$ pieces.

For each assumed increasing ordering $\pi \in S_d$\footnote{$S_d$ is the group of all permutations of $d$ objects.}, the corresponding hyperplane is given by
\begin{equation}\label{eq:hyperplane}
    h_\pi(w) = \frac{1}{d - 1}\sum_{i=1}^d (2i - d - 1) w_{\pi_i}.
\end{equation}

The hyperplane equation is derived by noting that the element $w_{\pi_i}$ appears with positive sign when compared to the $i - 1$ smaller weights $w_{\pi_1}, ..., w_{\pi_{i-1}}$ and with negative sign when compared to the $d - i$ larger weights $w_{\pi_{i + 1}}, ..., w_{\pi_d}$. We thus have an equation for each piece of $\R$.

The convexity of $\R$ implies that the tangent hyperplanes are below the graph of $\R$. Hence, $\R$ is the point-wise maximum of all hyperplanes. Writing $\R$ in this alternate way gives the equation below.
\begin{equation}\label{eq:r_alt}
    \R(w) = \max_{\pi \in S_d} h_\pi(w)
\end{equation}

\begin{example}
    Consider $d = 3$. Then,
    \begin{equation*}
        \R(w) = \frac{1}{2}\left(|w_1 - w_2| + |w_2 - w_3| + |w_1 - w_3|\right).
    \end{equation*}
    There are $3! = 6$ possible orderings of the weights which result in $6$ possible sign combinations for the terms inside the absolute values and $6$ hyperplanes. These are listed below.

    \begin{center}
    \begin{tabular}{|c|c|c|}
    \hline
    Ordering & Signs & Hyperplane \\
    \hline
    $w_1 < w_2 < w_3$ &  $- - -$ & $w_3 - w_1$ \\
    \hline
    $w_1 < w_3 < w_2$ & $- + +$ & $w_2 - w_3$ \\
    \hline
    $w_2 < w_1 < w_3$ & $+ - +$ & $w_1 - w_2$ \\
    \hline
    $w_2 < w_3 < w_1$ & $+ - -$ & $w_3 - w_2$ \\
    \hline
    $w_3 < w_1 < w_2$ & $- + -$ & $w_2 - w_1$ \\
    \hline
    $w_3 < w_2 < w_1$ & $+ + +$ & $w_1 - w_3$ \\
    \hline
    \end{tabular}
    \end{center}
    Note that some sign combinations are impossible, \eg $+ + -$, which would imply $w_1 > w_2$, $w_2 > w_3$, $w_3 > w_1$, which is impossible.

    By taking the point-wise maximum of the hyperplanes we can write $\R$ in the following equivalent form.
    \begin{equation*}
    \begin{split}
            \R(w) = \max \Big\{ & w_1 - w_3, w_1 - w_2, w_2 - w_3, \\
                           & w_2 - w_1, w_3 - w_2, w_3 - w_1 \Big\}
    \end{split}
    \end{equation*}

\end{example}

It is easy to see that the permutations that maximize $h_\pi(w)$ are the ones that sort $w$. 
When some weights are equal (\ie there is weight-sharing) there are multiple permutations that sort the weights. Let $\pisort(w)$ be the set of permutations that sort $w$, then, for all $\pi \in \pisort(w)$, we have $\R(w) = h_{\pi}(w)$. We refer to these $h_\pi$ as the \dfn{active hyperplanes} at $w$.

Using this new equation for $\R$, we can write its subdifferential as the convex hull of the gradient of active hyperplanes at $w$; See \cref{prop:subdiff_max}:
\begin{equation}\label{eq:subdiff_R}
    \partial \R(w) = \conv \bigcup_{\pi \in \pisort(w)} \nabla h_\pi.
\end{equation}

An intuitive description of the subgradients of $\R$ is that, for weights that are equal, you may assume any ordering and calculate a gradient. You may also take any convex combination of these gradients.

\section{PROXIMAL MAPPING OF $\R$}\label{sec:prox_R}

We describe a dynamical system for $w$, specifically, an ODE, such that following it for one time unit produces the proximal map of $\R$ at the initialization point of $w$, \ie $w(0)$. We construct the ODE such that $w$ follows the average of the negative gradient of the active hyperplanes. The ODE is given by
\begin{equation}\label{eq:prox_ode}
    \dot{w}(t) \defeq \frac{dw}{dt}(t) = \frac{1}{|\pisort(w(t))|} \sum_{\pi \in \pisort(w(t))} -\nabla h_\pi.
\end{equation}

Note that $-\dot{w}(t)$ is a subgradient of $\R$ at $w(t)$.

\begin{lemma}[Weight-sharing]\label{lemma:ws_preservation}
    The ODE preserves weight-sharing. Formally, for all $t$,
    \begin{equation*}
        w_i(t) = w_j(t) \implies \dot{w}_i(t) = \dot{w}_j(t).
    \end{equation*}
\end{lemma}

\begin{lemma}[Monotonic inclusion]\label{lemma:monotonic_inclusion}
    The ODE satisfies monotonic inclusion for the sets $\pisort$ and $\partial \R$. Specifically, for all $t_2 > t_1$:
    \begin{enumerate}[label=\textbf{\arabic*}.]
        \item $\pisort(w(t_1)) \subseteq \pisort(w(t_2))$
        \item $\partial \R(w(t_1)) \subseteq \partial \R(w(t_2))$
    \end{enumerate}
\end{lemma}

\begin{theorem}[Proximal mapping of $\R$]\label{thm:prox_R}
When $w$ follows the ODE,
\begin{equation*}w(1) = \prox_\R(w(0)).\end{equation*}
\end{theorem}

We will not be solving the ODE numerically. Instead, we will exploit certain properties of the ODE to obtain algorithms for exact computation of $w(1)$. These properties become clearer when we interpret $(w, \dot{w})$ as the state of a physical system of particles.

Think of each weight as a \textit{sticky} particle with unit mass moving in a 1-dimensional space. Furthermore, this system respects conservation of mass and momentum and Newton's laws of motion.

\begin{proposition}[Conservation of Momentum]\label{prop:conservation_of_momentum}
    Momentum is conserved in the ODE and is equal to $0$. More precisely,
    \begin{equation*}
        \sum_{i=1}^d \dot{w}_i = 0.   
    \end{equation*}
\end{proposition}

When two particles collide, they stick, due to \cref{lemma:ws_preservation}, and their masses and momenta are added. There are no other forces involved. These laws allow us to predict the future of the system, given the positions, masses, and velocities of all particles. They can therefore act as an intuitive replacement for the ODE.

If we initialize particle positions with the elements of $w$ and velocities with the elements of $-\nabla \R(w)$ (assuming all weights are distinct), then, the particle positions at time $t=1$ give us the proximal mapping of $\R$ at $w$, due to \cref{thm:prox_R}. To calculate the positions at time $t=1$ we need to identify collisions from $t=0$ to $t=1$.

We will be using the particles interpretation in the remainder of the paper. We denote the position, velocity, and mass of all particles by vectors $x, v \in \Re^d$, and $m \in \N^d$, respectively. We index particles by their ordering in space from left to right, \ie $x_i \leq x_{i+1}$. The vector $y \defeq x + v$ will be of use, so we assign it a symbol. It denotes the particle destinations after one time unit when no collisions occur. \citet{lin2019efficient} prove that applying isotonic regression to $y$ results in the proximal mapping of $\R$. In this paper, however, no knowledge of isotonic regression will be required.

\section{ALGORITHMS FOR $\prox_\R$}\label{sec:algorithms}

All of the algorithms in this section expect $w$ to have been sorted  and assigned to $x$ as a pre-processing step. Sorting has time complexity $O(d \log d)$ on a sequential machine. In the context of parallel algorithms, we assume access to a Parallel Random Access Machine (PRAM) with $p$ processors and $O(d)$ memory. The parallel time complexity of sorting is $O(\frac{d \log d}{p} + \log d)$~\citep{ajtai1983an, cole1988parallel}. The performance of sorting puts a ceiling on the performance that we can expect from algorithms for $\prox_\R$.

\subsection{Imminent Collisions Algorithm}
Matching the performance ceiling is easy on a sequential machine. In fact, apart from sorting, the rest of the algorithm can be implemented in \hl{$O(d)$ time}. The key idea is that the order of performing collisions does not matter, \ie the collision operation is associative. All that matters is which collisions occur. Thus, one can maintain a queue of all detected future collisions and iteratively pop the queue, perform a collision, check collisions for the new particle, and push newly detected collisions into the queue. To perform collisions efficiently, the vectors $x, v, m$ must be stored in doubly linked lists.

A collision is detected whenever $(y_i = x_i + v_i) \geq (y_{i + 1} = x_{i + 1} + v_{i + 1})$. We will refer to these collisions as \textit{imminent collisions} since they will occur regardless of other collisions.\footnote{Non-imminent collisions are harder to detect. For an if and only if condition for the collision of neighbouring particles, see \cref{sec:neighbour_collision}.} This algorithm is known in the literature as the \textit{pool adjacent violators algorithm}~\citep{best1990active}.

The algorithm above can be parallelized to some extent. We can detect all imminent collisions and perform them in parallel. The process is repeated until there are no more imminent collisions. We call this algorithm, \ica (pseudocode is provided in \cref{sec:extra_algorithms}).

\ica is fast when the total number of collisions is small but can be slow when a large cluster of particles forms. For example, consider
\begin{equation*}
y = \left[1, 1 - \frac{2}{d - 1}, 1 - \frac{4}{d - 1}, ..., 0, \epsilon, 2\epsilon, ..., \floor{\frac{d}{2}}\epsilon \right],
\end{equation*}
for some $0 < \epsilon \ll 1$. One can work out that all particles will eventually collide, but it will take $\frac{d}{2}$ imminent collisions rounds for all particles in the right half to collide since they are sorted. Therefore, in the worst case, the number of rounds is $O(d)$. Consequently, \ica has a worst-case \hl{parallel time complexity of $O(d)$} and might not benefit from parallelization.

\subsection{End Collisions Algorithm}
We continue our study of algorithms by noting some properties of the dynamical system. Let $\avg(u)$ denote the average of the elements of the vector $u$, and let $u_{i: j}$ for integers $i \leq j$ denote a vector of size $j - i + 1$ constructed from elements $i$ to $j$ (inclusive) of $u$. 

\begin{enumerate}
    \item When particles collide, they stick, so they cannot cross each other. Thus, the ordering of particles always remains the same.
    \item When particles $i$ and $i + 1$ collide, it must have been the case that $v_i > v_{i + 1}$. Hence, after collision, the velocity of particle $i$ increases in the negative direction, and the velocity of particle $i + 1$ increases in the positive direction. This also tells us that imminent collisions definitely occur, regardless of other collisions; A property that \ica relies on.
    \item Conservation of momentum implies that $\avg(y_{i: j})$ represents the final center of mass of particles $i$ to $j$ under the assumption that particles $i$ to $j$ form a closed system, that is, particles $i$ and $i - 1$ do not collide and particles $j$ and $j + 1$ do not collide.
\end{enumerate}

The proposition below tells us how many particles will collide with the left-most particle.

\begin{proposition}[Rightmost collision]\label{prop:rightmost_collision}
The rightmost collision of particle $1$ is with particle $\argminin_j \avg(y_{1: j})$. That is,
\begin{equation*}
    \text{particles } 1  \text{ and } i \text{ collide} \iff i \leq \argmin{j} \avg(y_{1: j}).
\end{equation*}
\end{proposition}

\cref{prop:rightmost_collision} suggests another algorithm. We repeatedly find all the particles that collide on one end, merge them, and then remove them, since they form a closed system that has no effect on the rest of the system. We call this algorithm, \eca (pseudocode is provided in \cref{sec:extra_algorithms}). The sequential version of this algorithm is known in the literature as the \textit{minimum lower sets algorithm}~\citep{best1990active} and runs in \hl{$O(d^2)$ time}.

If we let $c$ be the final number of particle clusters, then the \hl{parallel time complexity is $O(\frac{d c}{p} + c \log d)$} for \eca. The reason is that there are $c$ iterations, each one taking parallel time $O(\frac{d}{p} + \log d)$ due to computing $\argminin_{j} \avg(y_{1: j})$. This algorithm is, thus, very fast when the number of collisions is very large and a small number of particle clusters forms in the end. In the worst case, however, there are no collisions, and the algorithm takes \hl{$O(\frac{d^2}{p} + d \log d)$ parallel time}, which is slow. The algorithm can be made slightly faster by only considering particles that participate in an imminent collision, but the overall worst-case complexity remains the same.

\subsection{Search Collisions Algorithm}
So far we have seen rather trivial parallelizations of existing algorithms. We now describe a novel parallel divide and conquer algorithm that is fast for all inputs and has a depth of $O(\log^3 d)$.

We split the $d$ particles into two halves of size $\frac{d}{2}$ (suppose $d$ is even) and solve each half as a closed system in parallel recursively. Next, we need to consider the interactions between the two halves. There can only be an imminent collision at $(\frac{d}{2}, \frac{d}{2}+1)$. If there is no imminent collision, then there is no interaction between the halves and the solution has been found. Otherwise, $\frac{d}{2}$ and $\frac{d}{2} + 1$ will collide, and the new particle may collide with one of its neighbours, and so on. The chain of collisions causes this cluster of particles to grow until eventually it consumes some particles $i^\star \leq \frac{d}{2}$ to $j^\star > \frac{d}{2}$. The goal will be to find the ends $i^\star$ and $j^\star$ of the cluster so that we can perform the collisions and obtain the final solution.

We show that if we pick an index $i \geq i^\star$ from the left half and calculate its \dfn{rightmost collision}
\begin{equation}
\operatorname{rmc}(i) \defeq \argminin_{k \geq i} \avg(y_{i:k}), 
\end{equation}
then $\operatorname{rmc}(i) > \frac{d}{2}$. And if we pick an index $i < i^\star$, then $\operatorname{rmc}(i) = i < \frac{d}{2}$. We can therefore use the condition $\operatorname{rmc}(i) > \frac{d}{2}$ to perform a binary search on $i$ to find $i^\star$. The other end $j^\star$ is then obtained from $\operatorname{rmc}(i^\star)$, due to \cref{prop:rightmost_collision}.

\begin{theorem}\label{thm:search_collision}
    Suppose $(k, k+1)$ is the only imminent collision and particles $i^\star$ to $j^\star$ will collide. Then, for any index $i \leq k$:
    \begin{enumerate}[label=\textbf{\arabic*}.]
        \item $i < i^\star \implies \operatorname{rmc}(i) = i$
        \item $i \geq i^\star \implies \operatorname{rmc}(i) > k$
    \end{enumerate}
\end{theorem}

The binary search consists of $O(\log d)$ steps, each with a depth of $O(\log d)$, due to computing $\argminin_k \avg(y_{i:k})$. Thus, merging the solution of subproblems has depth $O(\log^2 d)$. There will be $\log d$ levels of merging in the divide and conquer binary tree, resulting in a total depth of $O(\log^3 d)$. Each step of binary search requires $O(d)$ operations in total on all subproblems. There are $\log^2 d$ binary search steps during the entire algorithm, giving a total work of $O(d \log^2 d)$. Altogether, we get a \hl{parallel time complexity of $O(\frac{d \log^2 d}{p} + \log^3 d)$}. We call this algorithm, \sca; See \cref{alg:search_collisions}.

\begin{algorithm}[!h]
\caption{Search Collisions Algorithm}\label{alg:search_collisions}
\small
\begin{algorithmic}
\Function{PerformCollisions}{$x, v, m, i, j$} 
    \State $m\_total \gets \operatorname{sum}(m[i:j])$
    \State $x[i] \gets \operatorname{sum}(x[i: j] \odot m[i: j]) / m\_total$ 
    \State $v[i] \gets \operatorname{sum}(v[i: j] \odot m[i: j]) / m\_total$ 
    \State $m[i] \gets m\_total$
    \State $m[i + 1: j] \gets 0$ 
\EndFunction
\vspace{0.5em}
\Function{RightmostCollision}{$x, v, m, i$}
    \State $d \gets \operatorname{size}(x)$
    \State $y\_cumsum \gets \operatorname{prefix\_sum}((x + v)[i:d] \odot m[i:d])$
    \State $avg \gets y\_cumsum / \operatorname{prefix\_sum}(m[i:d])$%
    \State $j \gets i + \operatorname{argmin}(avg) - 1$
    \State \Return j
\EndFunction
\vspace{0.5em}
\Function{Merge}{$x, v, m$}
    \State $d \gets \operatorname{size}(x)$
    \State $le \gets 1$
    \State $ri \gets \frac{d}{2} + 1$ 
    \For{$t=1$ to $\log_2 d - 1$} 
        \State $mid = \floor{(le + ri - 1) / 2}$
        \State $j \gets$ \Call{RightmostCollisions}{$x, v, m, mid$} 
        \If{$j \geq \frac{d}{2} + 1$}
            \State $ri \gets mid + 1$
        \Else
            \State $le \gets mid + 1$
        \EndIf
    \EndFor
    \State $j \gets$ \Call{RightmostCollision}{$x, v, m, le$}
    \State \textbf{if} $j > le$ \textbf{then} \Call{PerformCollisions}{$x, v, m, le, j$}
\EndFunction
\vspace{0.5em}
\Function{SearchCollisions}{$x, v, m$}
    \State \textbf{if} $\operatorname{size}(x) = 1$ \textbf{then} \Return
    \State Concatenate zeros at the ends of $x, v, m$ to make sizes be powers of $2$
    \State $d \gets \operatorname{size}(x)$
    \ParDo 
        \State \Call{SearchCollisions}{$x[1: \frac{d}{2}], v[1: \frac{d}{2}], m[1: \frac{d}{2}]$}
        \State \Call{SearchCollisions}{$x[\frac{d}{2}+1:], v[\frac{d}{2}+1:], m[\frac{d}{2}+1:]$}
    \End
    \State \Call{Merge}{$x, v, m$}
\EndFunction
\end{algorithmic}
\end{algorithm}

With $p \in \Omega(\frac{d}{\log d})$ processors, there is no asymptotic slowdown of the algorithm from a lack of processors and the running time is $O(\log^3 d)$. If there are too few processors, $p \in o(\log^2 d)$, then the work-inefficiency of \sca makes it worse than \ica. However, with modern GPUs, the number of processors does not become a bottleneck for any realistic problem size. For example, considering each of the 5,120 CUDA cores of NVIDIA V100 as a processor and ignoring coefficients in the running times, \sca has better worst-case performance for $d$ up to $2^{\sqrt{p}} = 2^{\sqrt{5120}} \approx 10^{21}$.

\begin{table}[!h]
    \centering
    \caption{Summary of the time-complexities of the parallel algorithms studied in this work.}
    \label{table:algs}
    \begin{tabular}{c c c} 
        \toprule
        \textbf{Parallel Algorithm} & \textbf{time complexity} \\ 
        \midrule
        Imminent collisions & $O(d)$  \\
        End collisions         & $O(\frac{d^2}{p} + d \log d)$ \\
        Search collisions & $O(\frac{d \log^2 d}{p} + \log^3 d)$ \\ 
        \bottomrule
    \end{tabular}
\end{table}

\section{REWINDING}\label{sec:rewinding}

It is common in machine learning to use $\ell_1$ regularization to obtain sparse solutions. But sparsity is essentially an $\ell_0$ constraint and $\ell_1$ is employed as a \textit{convex} relaxation of $\ell_0$. The $\ell_1$ norm does more than just sparsify the solution; It also makes all weights smaller. To rectify this side-effect, we propose \textit{rewinding}. The amount of rewinding will be parameterized by $\rho$.

Rewinding for $\ell_1$ amounts to moving weights that were not zeroed back towards where they started. When $\rho = 1$, the weights completely return to where they started and the side-effect is entirely removed. This is equivalent to the proximal mapping of $\ell_0$, which results in the iterative hard-thresholding algorithm~\citep{blumensath2008iterative}. If $\rho = 0$, there is no rewinding, and we recover the proximal mapping of $\ell_1$. 

\begin{algorithm}[!h]
\caption{$\operatorname{prox}_{\alpha \mathcal{R} + \beta \ell_1}$ with rewinding $\rho$.}\label{alg:prox}
\small
\begin{algorithmic}[1]
\Function{Prox}{$w, \alpha, \beta, \rho$}
    \State $d \gets \operatorname{size}(w)$
    \State $x, sort\_indices \gets \operatorname{sort}(w)$ 
    \State $v \gets \frac{\alpha}{d - 1} (d - 2 \cdot \operatorname{range}(d) + 1)$ \Comment{This is $-\alpha\nabla \R(x)$} 
    \State $m \gets \operatorname{ones\_array}(d)$ 
    \State $x, v, m \gets$ \Call{\raisebox{-0.45ex}{*}Collisions}{$x, v, m$} 
    \State $zero\_mask \gets |x + v| < \beta$ 
    \State $v \gets v - \beta \cdot \operatorname{sign}(x + v)$ 
    \State $x[zero\_mask] \gets 0$
    \State $v[zero\_mask] \gets 0$
    \State $x \gets x + (1 - \rho) \cdot v$ 
    \State $x \gets \operatorname{repeat\_interleave}(x, m)$
    \State $w[sort\_indices] \gets x$
\EndFunction
\end{algorithmic}
\end{algorithm}

The same concept can be applied to weight-sharing regularization.

\cref{alg:prox} applies weight-sharing and sparsity regularization with rewinding. Before line 11, $x$ satisfies weight-sharing and sparsity and $v$ contains the remaining displacement of the proximal mapping, that is, $\prox_{\alpha\R + \beta \ell_1}(w) = x + v$. However, adding $v$ to $x$ does not induce any additional sparsity or weight-sharing. We, thus, consider $v$ as the side-effect of convexifying sparsity and weight-sharing. To reduce the side-effect, we scale $v$ by $1 - \rho$ for rewinding.

As an example, for sparsity regularization with $\ell_1$, after processing constraints we have $x = w\indic{|w| > \beta}$ and $v = -\beta \indic{|w| > \beta}$. The $\rho$-rewinded update is then given by
\begin{equation}\label{eq:rewinding_l1}
x + (1 - \rho) v = (w - (1 - \rho) \beta)\indic{|w| > \beta}.
\end{equation}
This is rather similar to $\beta$-LASSO~\citep{neyshabur2020towards}, which uses a parameter $\beta > 1$ (different from our $\beta$) to perform more aggressive thresholding. The analog of their $\beta$ in our method is $\frac{1}{1 - \rho}$. However, their proposed method has a distinction from proximal methods in that the gradients of $\ell_1$ and the task-loss $\mathcal{L}$ are computed at the same point.

\section{EXPERIMENTS}\label{sec:experiments}

\subsection{MNIST on a Torus}

We consider a translation-invariant version of the MNIST dataset which we refer to as MNIST on a torus. We investigate whether our regularization can enhance the generalization of an MLP. While the answer to this question is positive, an invariant CNN model achieves higher accuracy compared to an MLP model regularized with $\mathcal{R}(w)$. However, CNNs make strong assumptions about the topology (i.e., pixel locality) and symmetry of data. When these assumptions are violated, they perform poorly, while the performance of our method is unaffected. To showcase this, we introduce a change in the dataset by performing a permutation on the pixels that destroys the translational symmetry of the images. In this setting, our results indicate that an MLP with weight-sharing regularization achieves the best accuracy, outperforming CNN by a significant margin.

\begin{figure}[!h]
    \centering
    \includegraphics[width=0.95\linewidth]{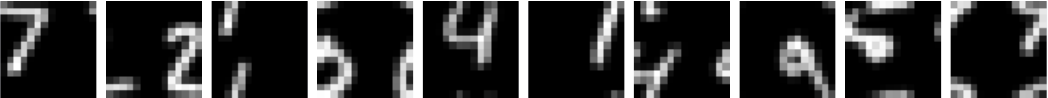}
    \caption{Sample digits from MNIST on a torus.}
    \label{fig:mnist_torus}
\end{figure}

\textbf{Dataset description.} We consider a version of the MNIST dataset~\citep{lecun1989backpropagation} in which the images lie on a torus. That is, pixels that cross the left boundary of the image reappear at the right boundary and vice versa. The same happens with the upper and bottom boundaries; See \cref{fig:mnist_torus}. Each image in the dataset is randomly translated. The task is, as usual, to classify the digits. In this setting, the task is invariant to the natural action of the group of 2D circular translations $\mathbb{Z}_k \times \mathbb{Z}_k$, where $k$ is the width and height of the input images. 

In our experiments, we also consider two variants of this dataset where pixels have been shuffled: (1) With symmetry but no locality. (2) With neither symmetry nor locality. Both of these variants are obtained by applying a fixed permutation of pixels on each image of the dataset. However, the order of this permutation w.r.t. the 2D circular translation is important. Performing the permutation before translation destroys the locality of pixels while maintaining symmetry whereas performing the operations in the opposite order removes symmetry as well as locality. 

\textbf{Experiment configuration.} To simulate data scarcity, we train with 60\% of the MNIST dataset. As a baseline, we use an invariant CNN model and an MLP that matches the architecture of the CNN, with weight-sharing and sparsity removed.

\textbf{Prediction asymmetry metric.}
We also provide a metric to measure the symmetry of each model. We define the \dfn{prediction asymmetry} of a model as the percentage of test examples for which there exist two translations of the input whose model predictions do not agree. We see that weight-sharing regularization has produced a more symmetric model w.r.t. this metric compared to other baselines.

\textbf{Results.} As shown in \cref{table:mnist_torus}, we observed that the unregularized fully connected neural network overfits the data. While it achieves 99.89\% training accuracy, test accuracy peaks at 92.40\%. Our weight-sharing regularization technique narrows the generalization gap, achieving 95.04\% accuracy. Weight-sharing regularization also reduces the prediction asymmetry of an MLP model significantly. However, an invariant CNN is by design the best model for the MNIST on torus dataset and achieves the highest accuracy.

\begin{table}[!h]
    \centering
    \caption{Results for the MNIST on torus experiment.}
    \label{table:mnist_torus}
    \begin{tabular}{l c c c} 
        \toprule
        \textbf{Model} & \textbf{Test Acc.} & \textbf{Pred. Asym.}\\ 
        \midrule
        \textbf{CNN}         & \textbf{98.82\%}  &  \textbf{0.0\%} \\
        CNN (-locality)      & 97.98\%  &  0.0\% \\
        CNN (-symmetry)      & 71.32\%  &  85.7\% \\
        \hline
        MLP                  & 92.40\% & 44.5\% \\
        MLP + $\ell_1$       & 93.57\% &  37.5\% \\
        \textbf{MLP + $\R$}  & \textbf{95.04\%} &  \textbf{31.6\%}\\ 
        \bottomrule
    \end{tabular}
\end{table}

We now consider the shuffled MNIST on torus datasets. The training of MLP-based models is completely unaffected in these datasets, and all of them achieve the exact same accuracy. As only the ordering of the input has changed, the MLP-based models can simply learn a permutation of the weights in the first layer to achieve the same performance as before. From \cref{table:mnist_torus} we observe that the absence of locality does not seem to significantly affect CNNs; however, the accuracy of this model drops by more than 25\% in the absence of symmetry. In this experiment, the advantage of weight-sharing regularization becomes evident, as it achieves the best accuracy among all models. Unlike CNN with its hard-coded bias, weight-sharing regularization aids MLP in dynamically learning the right weight-sharing pattern.

\subsection{CIFAR10}
We consider the task of training a shallow CNN and its corresponding fully connected network on CIFAR10~\citep{cifar} as suggested by ~\citep{neyshabur2020towards}. The author suggests a variant of the $\ell_1$ regularizer, $\beta$-LASSO, to learn convolution-like structures from scratch. Our experiments show that adding weight-sharing regularization to the fully connected neural network enables it to more effectively learn convolution-like patterns, which are recognized as optimal for vision data.
\begin{table}[!h]
    \centering
    \caption{Results for the CIFAR10 experiment.}
    \label{tab:cifar}
    \begin{adjustbox}{max width=0.48\textwidth}
    \begin{tabular}{ l c c c }
        \toprule
        \textbf{Model} & \textbf{Test Acc.} & \textbf{Sparsity} & \textbf{Weight Sh.}\\
        \midrule
        CNN            & 81.85\%  & 99.99\%  & 99.61\%   \\
        \hline
        MLP            & 69.20\%  & 0.0\%   & 91.05\% \\
        $\beta$-LASSO  & 71.37\%  & 95.33\% & 41.96\%  \\
        Our $\ell_1$  & 73.66\%  & 99.58\%  & 2.04\%\\
        $\R$ (Subgradient) & 69.96\%  & 0.0\%  & 2.98\% \\
        $\R$           & 73.50\%  & 0.0\%  & 99.96\% \\
        $\R$ + Our $\ell_1$  & \textbf{73.75}\%  & 99.71\% & 1.17\% \\
        \bottomrule
    \end{tabular}
\end{adjustbox}
\end{table}

\textbf{Experiment configuration.} We follow the setup of~\citep{neyshabur2020towards} and train a 3-layer fully connected network on a heavily augmented version of CIFAR10. The augmentations involve cropping and small rotations of the images, therefore, when training with weight-sharing regularization, we expect to see similar filters that are translated and tilted. We train the models for $400$ epochs with a batch-size of $512$ and a learning-rate of $0.1$. We performed a hyperparameter sweep for $\alpha$ and $\beta$ for each method and fixed $\rho$ to 0.98.

\textbf{Weight-sharing metric.} Our measure of weight-sharing in \cref{tab:cifar} is obtained by the formula $\frac{\text{\#non-zero distinct weights}}{\text{\#non-zero weights}}$. We do not include zeroed weights so that sparsity does not affect the measure.

\textbf{Results.} \cref{tab:cifar} shows that both our weight-sharing regularization and our adaptation of $\ell_1$ regularization improve test accuracy compared to MLP and $\beta$-LASSO. We note that regularization with $\R$ achieves a weight-sharing measure that is closely aligned with that of the CNN network. As shown in the last row of \cref{tab:cifar}, we also observe that incorporating $\ell_1$ regularization into weight-sharing is likely to drive the shared weights towards zero. \\
\cref{tab:cifar} also highlights the importance of using the proximal gradient descent algorithm compared to the subgradient method, which fails to achieve weight-sharing when trained with $\R$.

\begin{figure}[!h]
    \centering
    \includegraphics[width=0.95\linewidth]{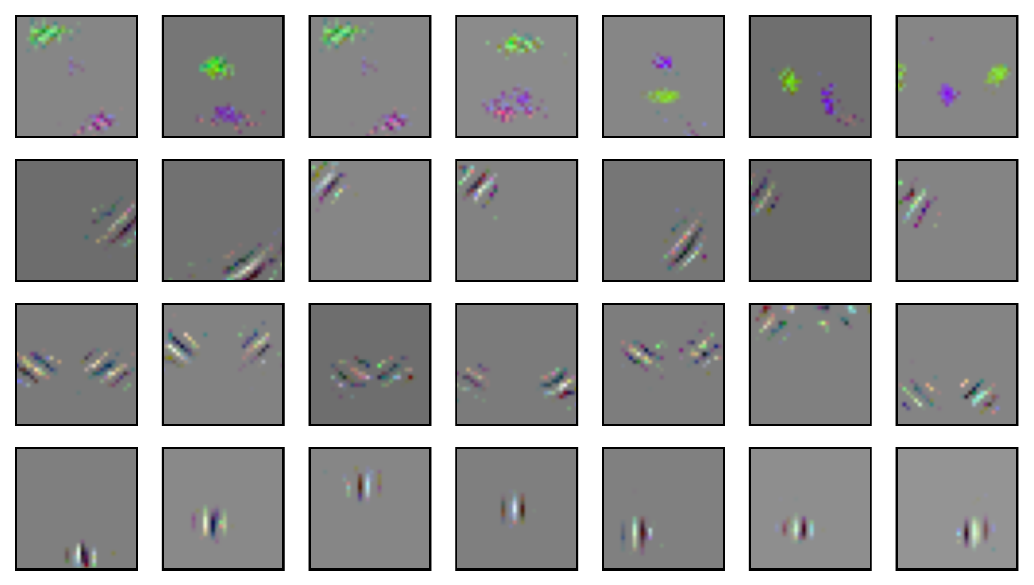}
    \caption{The emergence of learned convolution-like filters in a fully connected network with weight-sharing regularization on CIFAR10.}
    \label{fig:learned_conv_selected}
\end{figure}

\begin{figure}[!h]
    \centering
    \includegraphics[width=0.9\linewidth]{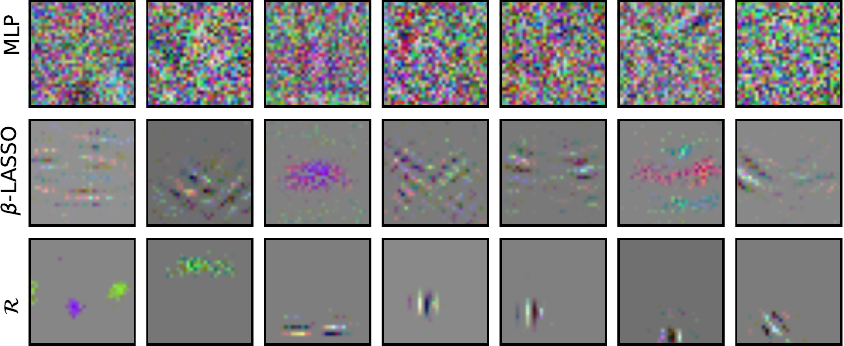}
    \caption{A random sample of learned filters from the 256 filters with the highest number of non-zero weights.} 
    \label{fig:learned_conv_random}
\end{figure}

By visualizing the rows of the first layer's weight matrix, we notice some clusters of similar and convolution-like patterns. Some of these similar patterns are demonstrated in \cref{fig:learned_conv_selected}. While~\citep{neyshabur2020towards} recreated sparsity of patterns, we also observe sparse weights with weight-sharing among them. These local patterns are translated across the input image to simulate convolution in the fully connected neural network.

\textbf{Practical remarks.} By comparing the training times of various algorithms, we observe that \ica is faster for weights encountered during training compared to \sca. We observed that \ica requires $\sim$500 rounds of iterations for $d \approx 10^8$ while in the worst case, it would require $d$ rounds. We, therefore, recommend using \sca as a backup in case \ica gets stuck on a difficult input. Such a setup will \textit{guarantee} fast training. Further details regarding runtimes can be found in \cref{app:run_time_compare}.

\section{FUTURE WORK}\label{sec:future_work}

Even though $\prox_\R$ greatly benefits from parallelization, training with weight-sharing regularization is still several times slower than training without. To bring down the cost of training, it might be possible to reuse computation from the previous training step. Is it possible to design an algorithm that starts with a candidate weight-sharing, and then edits it to obtain the correct weight-sharing?

Weight-sharing regularization with rewinding introduces new hyper-parameters $\alpha$ and $\rho$ that have to be tuned. The hyper-parameters can be fixed or they can vary during training. There is, therefore, a large space of possibilities to explore, and it is currently unclear what kind of training scheme is best.

A possible generalization of $\R$ is given by

\begin{equation}
    \R_{\mat{C}}(w) = \frac{1}{d - 1}\sum_{i < j} \mat{C}_{i,j} \norm{w_i - w_j}_2,
\end{equation}

where $\mat{C} \in \Re_+^{d \times d} > 0$ is a symmetric matrix and $w \in \Re^{d \times k}$. We have studied $\mat{C} = \mat{1}$, where all pairs of weights are encouraged to be equal, and $k = 1$, where each weight is a scalar. As an example, when
\begin{align*}
        \mat{C}_{i, j} =
        \begin{cases}
            1 & \abs{i - j} = 1\\
            0 & \text{otherwise}
        \end{cases},
    \end{align*}
we retrieve the fused lasso regularizer~\citep{tibshirani2005sparsity}. Does there exist fast parallel algorithms for $\prox_{\R_{\mat{C}}}$ as well?

Lastly, we leave the study of rewinding and its properties for future work.

\section{CONCLUSION}\label{sec:conclusion}

We have proposed ``weight-sharing regularization'' via $\R$ and obtained its proximal mapping by formulating it in terms of the solution to an ODE. By analogy to a physical system of interacting particles, we developed a highly parallel algorithm suitable for modern GPUs. Our experiments indicate that this method can be applied to the training of weight-sharing regularized deep neural networks using proximal gradient descent to achieve better generalization in the low-data regime. While our preliminary results are encouraging, 
there's room for further research and improvement in this domain.

\subsubsection*{Acknowledgments}
We thank Juan Duque, Juan Ramirez, Gauthier Gidel, Quentin Bertrand, and Reza Babanezhad for their feedback on an initial draft of this paper. This research was enabled in part by computing resources provided by Mila (mila.quebec). This work was supported by the Canada CIFAR AI Chair Program. Simon Lacoste-Julien is a CIFAR Associate Fellow in the Learning Machines \& Brains program.

\bibliography{paper}
\bibliographystyle{paper}

\section*{Checklist}

 \begin{enumerate}

 \item For all models and algorithms presented, check if you include:
 \begin{enumerate}
   \item A clear description of the mathematical setting, assumptions, algorithm, and/or model. \textit{Yes. We provide precise pseudocode for all algorithms.}
   \item An analysis of the properties and complexity (time, space, sample size) of any algorithm. \textit{Yes. time complexity is provided for all algorithms and space is always $O(d)$.}
   \item (Optional) Anonymized source code, with specification of all dependencies, including external libraries. \textit{Yes. Our code is available on \href{https://github.com/motahareh-sohrabi/weight-sharing-regularization}{github}.}
 \end{enumerate}

 \item For any theoretical claim, check if you include:
 \begin{enumerate}
   \item Statements of the full set of assumptions of all theoretical results. \textit{Yes.}
   \item Complete proofs of all theoretical results. \textit{Yes. All proofs are provided in \cref{sec:proofs}.}
   \item Clear explanations of any assumptions. \textit{Yes.}
 \end{enumerate}

 \item For all figures and tables that present empirical results, check if you include:
 \begin{enumerate}
   \item The code, data, and instructions needed to reproduce the main experimental results (either in the supplemental material or as a URL). \textit{Yes. We have released our source code on \href{https://github.com/motahareh-sohrabi/weight-sharing-regularization}{github} to facilitate the reproduction of all experimental results.}
   \item All the training details (e.g., data splits, hyperparameters, how they were chosen). \textit{Yes. Some training details are provided in the text and the rest is provided in the appendix.}
   \item A clear definition of the specific measure or statistics and error bars (e.g., with respect to the random seed after running experiments multiple times). \textit{No. The measures that we use are mentioned, but we do not provide error bars.}
   \item A description of the computing infrastructure used. (e.g., type of GPUs, internal cluster, or cloud provider). \textit{Yes. We use NVIDIA GPUs such as V100 and A100.}
 \end{enumerate}

 \item If you are using existing assets (e.g., code, data, models) or curating/releasing new assets, check if you include:
 \begin{enumerate}
   \item Citations of the creator If your work uses existing assets. \textit{Yes. We have cited relevant publications for the MNIST and CIFAR10 datasets.}
   \item The license information of the assets, if applicable. \textit{No. These datasets are standard in machine learning.}
   \item New assets either in the supplemental material or as a URL, if applicable. \textit{Not Applicable.}
   \item Information about consent from data providers/curators. \textit{Not Applicable.}
   \item Discussion of sensible content if applicable, e.g., personally identifiable information or offensive content. \textit{Not Applicable.}
 \end{enumerate}

 \item If you used crowdsourcing or conducted research with human subjects, check if you include:
 \begin{enumerate}
   \item The full text of instructions given to participants and screenshots. \textit{Not Applicable.}
   \item Descriptions of potential participant risks, with links to Institutional Review Board (IRB) approvals if applicable. \textit{Not Applicable.}
   \item The estimated hourly wage paid to participants and the total amount spent on participant compensation. \textit{Not Applicable.}
 \end{enumerate}

 \end{enumerate}

\clearpage
\onecolumn
\appendix
\section{BACKGROUND}\label{sec:background}

\subsection{Optimization}

Since $\R$ is not differentiable everywhere, we will be using concepts that generalize the gradient to non-differentiable functions, namely, the subgradient and subdifferential.

\begin{definition}[Subgradient]
    $g$ is a subgradient of a convex function $f$ at $x \in \dom f$ if
    \begin{equation*}
        f(y) \geq f(x) + g \cdot (y-x) \qquad \forall x \in \dom f.
    \end{equation*}
\end{definition}

\begin{definition}[Subdifferential]
    The set of subgradients of $f$ at the point $x$ is called the subdifferential of $f$ at $x$, and is denoted $\partial f(x)$.
    \begin{equation*}
        \partial f(x) \defeq \big\{ g \;|\; f(y) \geq f(x) + g \cdot (y-x),\ \forall y \in \dom f \big\}
    \end{equation*}
\end{definition}

The subdifferential is a closed convex set. Each point $y \in \dom{f}$ puts a linear inequality constraint on $g$ which results in a closed convex set. The result follows from the fact that the intersection of any collection of closed convex sets is closed and convex.

In the text, we rewrite $\R$ as a $\max$ and use the following proposition~\citep[Proposition 2.3.12]{clarke1990optimization} to obtain its subdifferential.

\begin{proposition}[Subdifferential of $\max$]\label{prop:subdiff_max}
    Let
    \begin{equation*}
    f(x) = \max\ \{f_1(x), ..., f_k(x)\},
    \end{equation*}
    where $f_i(x)$ are differentiable convex functions. A function $f_i$ is \dfn{active} at $x$ if $f(x) = f_i(x)$. Let $I(x)$ denote the set of active functions at $x$. Then,
    \begin{equation*}
        \partial f(x) = \conv \bigcup_{i \in I(x)} \nabla f_i(x),
    \end{equation*}
    where $\conv$ denotes the convex hull.
\end{proposition}

Proximal optimization methods rely on the proximal mapping, defined below.

\begin{definition}[Proximal mapping]
The proximal mapping of a closed convex function $f$ is
\begin{equation*}
\prox_f(x) \defeq \argmin{u}\left(f(u)+\frac{1}{2}\norm{u-x}_2^2\right).
\end{equation*}
\end{definition}

Since $u \mapsto f(u)+\frac{1}{2}\norm{u-x}_2^2$ is closed and strongly convex, the proximal map exists and is \textit{unique}.

The \textit{optimality condition} for minimizing a non-differentiable convex function $g$ is $0 \in \partial g(x^\star)$. Applying this fact to the proximal mapping tells us that
\begin{equation}\label{eq:prox_opt}
u^\star = \prox_f(x) \iff x - u^\star \in \partial f(u^\star).
\end{equation}

Intuitively, the result says that after taking a proximal step from $x$ to $u^\star$, there exists a subgradient at $u^\star$ that takes you back to $x$.

Gradient descent can be seen as repeated application of the proximal mapping to a local linearization of the objective function. More explicitly, let $f_y$ denote the linearization of $f$ at $y$:
\begin{equation}
f_y(x) \defeq f(y) + (x - y) \cdot \nabla f(y).
\end{equation}

Then, a gradient descent step on $f$ with step-size $\eta$ can be written as
\begin{equation}\label{eq:gd_prox}
x^{(t + 1)} = x^{(t)} - \eta \nabla f(x^{(t)}) = \prox_{\eta f_{x^{(t)}}}(x^{(t)}).    
\end{equation}

Whenever the objective function is the sum of a differentiable component $f$ and a non-differentiable component $h$ for which we can compute a proximal mapping, an optimization algorithm known as \textit{proximal gradient descent} can be used. Similar to gradient descent, it uses a linearization of $f$ at each time-step, however, $h$ is not linearized. Specifically, the proximal mapping of $f_x + h$ is used. The algebraic manipulations below show that the algorithm alternates between gradient descent steps on $f$ and proximal mappings of $h$.

\begin{align*}
\prox_{f_x + h}(x) &= \argmin{u}\bigg( f(x) + (u - x) \cdot \nabla f(x) + h(u) +\frac{1}{2}\norm{u-x}_2^2 \bigg) \\
&= \argmin{u}\bigg( h(u) + \frac{1}{2} \norm{u - (x - \nabla f  (x))}_2^2 \bigg) \\
&= \prox_h(x - \nabla f  (x))
\end{align*}

Therefore, proximal gradient descent is an efficient algorithm whenever $\prox_h$ can be computed efficiently. Similar to \cref{eq:gd_prox}, a step-size can be included to get the update equation
\begin{equation}\label{eq:pgd}
    x^{(t + 1)} = \prox_{\eta h}(x^{(t)} - \eta\nabla f(x^{(t)})).
\end{equation}
\begin{example}
    As an example, let us obtain the proximal mapping of $f(x) = \abs{x}$ for $x \in \Re$. The subdifferential of $f$ is given by
    \begin{align*}
        \partial f(x) =
        \begin{cases}
            \{1\}   & x > 0\\
            [-1, 1] & x = 0\\
            \{-1\}  & x < 0
        \end{cases}.
    \end{align*}
    The optimality condition is $0 \in \partial f(u) + u - x$, which gives us
    \begin{align*}
        0 \in
        \begin{cases}
            \{1 + u - x\}   & u > 0\\
            [-1, 1] + u - x & u = 0\\
            \{-1\} + u - x  & u < 0
        \end{cases}.
    \end{align*}
    We then solve for $u$. The three cases can be solved independently.
    \begin{align*}
        u =
        \begin{cases}
            x - 1 & x > 1\\
            0     & x \in [-1, 1]\\
            x + 1 & x < -1
        \end{cases}    
    \end{align*}
    In $\Re^d$, element-wise application of the above mapping gives us the proximal mapping of the $\ell_1$ norm, also commonly known as \textit{soft-thresholding}.
\end{example}

\subsection{Parallel Computation}

In parallel computation, we assume access to a machine with multiple processors that is able to perform multiple operations at the same time. With access to such a parallel machine with $p$ processors, one could potentially gain a $p \times$ speedup of any sequential algorithm. The actual performance improvement depends on the structure of the algorithm and the inter-dependencies among the operations of the algorithm. For some problems, new algorithms have to be designed to be able to leverage parallelism, such as what we do in this work.

The commonly used framework for analyzing parallel algorithms is the \textit{work-depth model}. Let $n$ denote the size of the problem. The \dfn{depth} $D(n)$ of an algorithm, also known as the \dfn{span}, is the greatest number of basic operations (with constant fan-in) that have to be performed \textit{sequentially} by the algorithm due to data dependencies. Or, in other words, the length of the critical path of the algorithm's circuit (\textit{a.k.a.,} computation graph). The \dfn{work} $W(n)$ of an algorithm is defined as the \textit{total} number of basic operations that the algorithm has to perform.

Since no more that $p$ operations can be performed at any time, an algorithm will require at least $W(n)/p$ rounds of parallel operations. On the other hand, an algorithm requires at least $D(n)$ rounds of parallel operations, by definition. It turns out that the running time $T(n)$ of a parallel algorithm is bounded by the sum of these lower-bounds:
\begin{equation}
T(n) = O\Big(\frac{W(n)}{p} + D(n)\Big).
\end{equation}

There are certain fundamental parallel operations that are commonly used in parallel computation. We describe the ones that we use along with their parallel time complexities.

\begin{itemize}
    \item \texttt{map} applies a given function to each element of the input list in parallel. The time complexity of \texttt{map} is \hl{$O(\frac{n}{p} + 1)$} as each processor can independently process one element of the input in constant time.

    \item \texttt{reduce} combines all elements of the input list using an associative operation in a hierarchical manner. We can perform a binary reduction in \hl{$O(\frac{n}{p} + \log n)$} time. In each step, half of the processors perform the associative operation with the result of their neighbouring processor, reducing the problem size by half.

    \item \texttt{scan} computes a running sum (or any other associative operation) of the input list. An algorithm, known as Hillis-Steele scan, works by having each processor \(i\) compute the sum of the input element at \(i\) and all preceding elements at distances of \(2^j\) for \(j=0\) to \(\log(n) - 1\), thus requiring $\log n$ steps to complete. There exists a work-efficient algorithm for \texttt{scan} that runs in time \hl{$O(\frac{n}{p} + \log n)$}.
\end{itemize}

\section{COMPUTING $\R$}

Proximal gradient descent does not require computing $\R$, nonetheless, it's interesting to analyze the complexity of computing $\R$ and contrast it to that of $\prox_\R$.

At first glance, computing $\R(w)$ seems to require quadratic time in the size of $w$, \ie $O(d^2)$. However, if we sort the weights, $\R$ can be rewritten in a form that is faster to compute. Note that \textit{$\R$ is permutation invariant}. Let $x$ denote the sorted weight vector. Then,
\begin{align*}
\R(w) = \R(x)
      = \frac{1}{d - 1}\sum_i \sum_{j < i} x_i - x_j
      = \frac{1}{d - 1}\sum_i \Big( (i - 1)x_i - \sum_{j < i} x_j \Big).
\end{align*}
Let us define a new vector $s$ containing the prefix sum of $w$, that is, $s_i \defeq \sum_{j < i} x_j$. This vector can be computed in time $O(d)$. Next, $\R(w)$ can be computed with an additional $O(d)$ time. Sorting is, thus, the bottleneck of the algorithm, resulting in an \hl{$O(d \log d)$ time} algorithm for $\R$.

This algorithm can greatly benefit from parallelization. Calculating $\R(x)$ only requires two prefix sums, a vector multiplication and subtraction, and a sum, which take $O(\frac{d}{p} + \log d)$ parallel time in total. Additionally, sorting can be performed in $O(\frac{d \log d}{p} + \log d)$ parallel time, giving us an \hl{$O(\frac{d \log d}{p} + \log d)$ parallel time} algorithm for computing $\R$.

\section{PROOFS}\label{sec:proofs}

\subsection{Proof of \cref{lemma:ws_preservation}}

\begin{proof}
Let $\sigma \in S_d$ be the permutation that only swaps indices $i$ and $j$. When $w_i = w_j$, the set $\pisort(w)$ is the same as the set $\{ \sigma \pi \st \pi \in \pisort(w)\}$. Therefore,
\begin{align*}
    \sum_{\pi \in \pisort(w)} \left( \nabla h_\pi \right)_i = \sum_{\pi \in \pisort(w)} \left( \nabla h_{\sigma \pi} \right)_i = \sum_{\pi \in \pisort(w)} \left( \nabla h_{\pi} \right)_j.
\end{align*}
Dividing by $-\abs{\pisort(w)}$ gives the desired result.
\end{proof}

\subsection{Proof of \cref{lemma:monotonic_inclusion}}

\begin{proof}
    Let $\pi$ be a permutation that sorts the weights. Assuming all weights $w_i \in \Re$ to be distinct, we have
    \begin{equation}\label{eq:w_sort}
        w_{\pi_1} < w_{\pi_2} < ... < w_{\pi_{d-1}} < w_{\pi_d}.
    \end{equation}
    If some of these weights are equal, some $<$ signs should be replaced with $=$.

    The weights $w_i(t)$ change \textit{smoothly} in time according to the ODE, and when two weights meet, they stay equal forever, due to \cref{lemma:ws_preservation}. Therefore, as time progresses, we may only have that, in \cref{eq:w_sort}, some $<$ signs turn to $=$. Clearly, previous sorting permutations are still valid while new ones are added to $\pisort$, hence, $\pisort$ satisfies monotonic inclusion. It immediately follows that $\partial \R$ also satisfies monotonic inclusion by \cref{eq:subdiff_R}.
\end{proof}

\subsection{Proof of \cref{thm:prox_R}}

\begin{proof}
Note that $\dot{w}$ only changes when new weights become equal. Now suppose $w$ follows the negative direction of subgradients $g_1, ..., g_k$ for $t_1, ..., t_k$ time units such that $\sum_{i=1}^k t_i = 1$. Due to monotonic inclusion of the subdifferential of $\R$, we have
\begin{equation*}
    g_1, ..., g_k \in \partial \R(w(1)).
\end{equation*}
The negative displacement of $w$ (\ie $(w(0) - w(1)$) is equal to $\sum_{i=1}^k t_i g_i$, which is a convex combination of subgradients at $w(1)$. Therefore, the negative displacement is in the subdifferential of $w(1)$, which implies that $w(1)$ is the result of the proximal mapping.
\end{proof}

\subsection{Proof of \cref{prop:conservation_of_momentum}}

\begin{proof}
Note that, for all tangent hyperplanes $h$ of $\R$,
\begin{equation}\label{eq:grad_sum}
\sum_{i=1}^d (\nabla h)_i = \frac{1}{d - 1}\sum_{i=1}^d (2i - d - 1) = 0.
\end{equation}
Since any subgradient $g$ of $\R$ is a convex combination of the gradient of tangent hyperplanes, $\sum_i g_i = 0$. As $-\dot{w}$ is a subgradient, the same fact applies.
\end{proof}

\subsection{Proof of \cref{prop:rightmost_collision}}

\begin{proof}
    $\Leftarrow:$ Let us focus on where particle $1$ will end up. Let $y'_j$ denote the final position of particle $j$. Recall that $\avg(y_{1: j})$ represents the final center of mass of particles $1$ to $j$ under the assumption that particles $j$ and $j + 1$ do not collide. If they do collide, the velocity of particle $j$ increases in the negative direction, and so, the final center of mass will only be smaller.

    Since particle $1$ is always the left-most particle, we have $y'_1 \leq \avg(y_{1: j})$ for all $j$. Minimizing over $j$ we get $y'_1 \leq \min_j \avg(y_{1: j})$. On the other hand, $y'_1$ is equal to the the final center of mass of \textit{some} group of neighbouring particles that all collide, therefore,
    \begin{equation}
    y'_1 = \min_j \avg(y_{1: j}),    
    \end{equation}
    and particle $1$ collides with all particles up to $\argminin_j \avg(y_{1: j})$.

    $\Rightarrow:$ Only these particles collide with particle $1$ because otherwise particle $1$ would end up at a different greater position.
\end{proof}

\subsection{Neighbour Collision}\label{sec:neighbour_collision}

The following corollary of \cref{prop:rightmost_collision} provides an if and only if condition for the collision of two \textit{neighbouring} particles.
\newpage
\begin{corollary}[Neighbour collision]
\begin{equation*}
     \text{Particles } i \text{ and } i + 1 \text{ collide}
\iff \max_{j \leq i} \avg(y_{j: i}) \geq \min_{j \geq i + 1} \avg(y_{i+1: j})
\end{equation*}
\end{corollary}

\begin{proof}
    Considering particles $1$ to $i$ as a closed system then particle $i$ is heading to $\max_{j \leq i} \avg(y_{j: i})$. Similarly, considering particles $i + 1$ to $d$ as a closed system, particle $i + 1$ is heading to $\min_{j \geq i + 1} \avg(y_{i+1: j})$. A collision occurs if and only if the destinations cross each other.
\end{proof}

\subsection{Proof of \cref{thm:search_collision}}

\begin{proof}
\textbf{1.} If $i < i^\star$, we know that particle $i$ does not participate in any collisions, so its rightmost collision is itself.

\textbf{2.} Recall that a chain of collisions starts from $(k, k + 1)$ and extends outwards. By \cref{prop:rightmost_collision}, $\operatorname{rmc}(i)$ gives us the rightmost collision of particle $i$ when assuming it is the first particle. Therefore, we should assume that particles less than $i$ don't exist. Without these particles, the final cluster of collided particles will still include $i$ but will end at some $j'$, where $k + 1 \leq j' \leq j^\star$. The rightmost collision of $i$ is, therefore, $j'$, which satisfies $j' > k$.
\end{proof}

\subsection{Combining Weight-Sharing and Sparsity}\label{sec:prox_R_l1}

\begin{proposition}[Proximal mapping of $\R+\beta \ell_1$]\label{prop:prox_R_plus_l1}
\begin{equation*}
\prox_{\R + \beta \ell_1}(x) = \prox_{\beta \ell_1}(\prox_\R(x))
\end{equation*}
\end{proposition}

\begin{proof}
    We start with several definitions that will aid the readability of the proof.

    \begin{figure}[!h]
    \centering
    \begin{tikzpicture}
        \node at (0, 0) (x) {$x$};
        \node at (0,-1.5) (x_prime) {$x'$};
        \node at (0,-3) (x_double_prime) {$x''$};
        \draw[->] (x) -- node[right] {$\Delta_\R' \defeq x - x' \in \partial \R(x')$} (x_prime);
        \draw[->] (x_prime) -- node[right] {$\Delta_{\beta \ell_1}'' \defeq x' - x'' \in \partial(\beta \ell_1)(x'')$} (x_double_prime);
        \node[right=0.2cm] at (x) {};
        \node[right=0.2cm] at (x_prime) {$\defeq \prox_\R(x)$};
        \node[right=0.2cm] at (x_double_prime) {$\defeq \prox_{\beta \ell_1}(\prox_\R(x))$};
        \end{tikzpicture}
    \end{figure}

    We must show that $x - x'' \in \partial (\R + \beta \ell_1)(x'')$, or equivalently, $\Delta_\R' + \Delta_{\beta \ell_1}'' \in \partial \R(x'') + \partial (\beta \ell_1)(x'')$. This is true if $\Delta_\R' \in \partial \R(x'')$. Since $x''$ is the result of soft-thresholding $x'$, and since soft-thresholding preserves order and weight-sharing, similar to \cref{lemma:monotonic_inclusion}, we get monotonic inclusion for $\partial \R$, \ie $\forall x \in \Re^d: \partial \R(x) \subseteq \partial \R(\prox_{\beta \ell_1}(x))$. Hence, $\Delta_\R' \in \partial \R(x') \subseteq \partial \R(x'')$, which completes the proof.
\end{proof}

\subsection{Counterexample for \citet{kearsley1996approach} }\label{app:counterexample}
\citet{kearsley1996approach} propose a parallel algorithm for isotonic regression which we show to be \textit{incorrect} with a simple counterexample. Let
\begin{equation*}
y = [0.7, 1, 0.9, 0.99]    
\end{equation*}
in the isotonic regression problem of \cref{eq:iso_regr}. Their algorithm proposes to average values $1, 0.9, 0.99$ and arrives at the solution
\begin{equation*}
x = [0.7, 0.96\bar{3}, 0.96\bar{3}, 0.96\bar{3}],    
\end{equation*}
while the correct solution only averages $1$ and $0.9$ to get
\begin{equation*}
x^\star = [0.7, 0.95, 0.95, 0.99].
\end{equation*}

\section{PSEUDOCODES}\label{sec:extra_algorithms}

The functions \textsc{RightmostCollision} and \textsc{PerformCollisions} are defined in \cref{alg:search_collisions}. We observed that \cref{alg:imminent_collisions_v2} is the fastest in practice for training neural networks with weight-sharing regularization. It performs an additional \texttt{prefix\_sum} to detect successive imminent collisions, then performs them all at once in a single round. The \texttt{prefix\_sum} adds a factor of $O(\log d)$ to the depth, however, since we've empirically observed that there aren't too many rounds in practice, the overhead is negligible.

\noindent 
\begin{minipage}{.48\linewidth}

\begin{algorithm}[H]
\caption{Imminent collisions algorithm (v2)}\label{alg:imminent_collisions_v2}
\begin{algorithmic}
\Function{CollisionsRound}{$x, v, m$}
    \State $n \gets \operatorname{size}(x)$
    \State $c \gets \operatorname{zeros\_array}(n)$
    \State $c[1:] \gets (x + v)[:-1] < (x + v)[1:]$
    \State $i \gets \operatorname{prefix\_sum}(c)$
    \State $x \gets \operatorname{scatter}(x \odot m, i)$ 
    \State $v \gets \operatorname{scatter}(v \odot m, i)$
    \State $m \gets \operatorname{scatter}(m, i)$
    \State $x \gets x / m$
    \State $v \gets v / m$
    \State \Return $x, v, m$
\EndFunction
\vspace{0.5em}
\Function{ImminentCollisions}{$x, v, m$}
    \Repeat
        \State $d \gets \operatorname{size}(x)$
        \State $x, v, m \gets$ \Call{CollisionsRound}{$x, v, m$}
    \Until{$d = \operatorname{size}(x)$}
    \State \Return $x, v, m$
\EndFunction
\end{algorithmic}
\end{algorithm}

\end{minipage}
\begin{minipage}{.515\linewidth}
    
\begin{algorithm}[H]
\caption{Imminent collisions algorithm (v1)}\label{alg:imminent_collisions_v1}
\begin{algorithmic}
\Function{CollisionsRound}{$x, v, m$}
    \State $d \gets \operatorname{size}(x)$
    \ParForAll{odd indices $i < d$}
        \If{$(x + v)[i] > (x + v)[i + 1]$} 
            \State \Call{PerformCollisions}{$x, v, m, i, i + 1$}
        \EndIf
    \EndFor 
    \State \Return $x[m > 0], v[m > 0], m[m > 0]$ 
\EndFunction
\vspace{2.9em}
\Function{ImminentCollisions}{$x, v, m$}
    \Repeat
        \State $d \gets \operatorname{size}(x)$
        
        \State $x, v, m \gets$ \Call{CollisionsRound}{$x, v, m$} 
        \State $x, v, m \gets$ \Call{CollisionsRound}{$x[2:], v[2:], m[2:]$} 
    \Until{$d = \operatorname{size}(x)$}
    \State \Return $x, v, m$
\EndFunction
\end{algorithmic}
\end{algorithm}

\end{minipage}

\begin{algorithm}[!h]
\caption{End collisions algorithm}\label{alg:end_collisions}
\begin{algorithmic}
\Function{EndCollisions}{$x, v, m$}
    \State $d \gets \operatorname{size}(x)$
    \State $i \gets 1$
    \While{$i < d$}
        \State $j \gets$ \Call{RightmostCollision}{$x, v, m, i$}
        \State \Call{PerformCollisions}{$x, v, m, i, j$}
        \State $i \gets j + 1$
    \EndWhile
\EndFunction
\end{algorithmic}
\end{algorithm}

\section{EXPERIMENTS}\label{sec:exp_details}
\subsection{MNIST on Torus}

The CNN baseline is designed to be invariant to circular 2D translations. The architecture consists of two circular convolutional layers, outputting 32 and 64 channels, respectively, followed by global average pooling and a fully connected layer. Note that circular convolution is equivariant to circular translations, while global pooling is invariant, resulting in an invariant model.

The fully connected networks were obtained from the CNN baseline by replacing convolutional layers with linear layers with matching input and output sizes. In order to see convolution-like behaviour in the first two layers of the MLPs, we apply weight-sharing regularization to each of these layers separately.

We trained all the networks for 200 epochs using an initial learning-rate of 0.1 which is cosine annealed to 0. We also used a momentum of 0.9 to achieve near-zero training error for all models within 200 epochs. As a result, all the models attained more than 99\% training accuracy, except for our method when using dataset ratios of 0.8 and 1.0, where it achieved more than 98\% training accuracy. To choose the weight-sharing parameter $\alpha$ for each layer, we ran a sweep on values 0.01, 0.001, 0.0001, and 0.00001. The selected parameters were 0.001 for the first layer and 0.0001 for the second layer. For the $\ell_1$ coefficient $\beta$, we selected 0.0001 after a hyperparameter search. We don't use rewinding in these experiments.

Additionally, we use a modification to the standard $\ell_1$ proximal update rule as proposed by \citet{neyshabur2020towards}, which is to multiply the learning-rate in $v$ in line 11 of \cref{alg:prox}, instead of multiplying it in the coefficients $\alpha$ and $\beta$. This modification improves validation accuracy in practice.

\begin{figure}[!h]
    \centering
    \includegraphics[width=0.47\linewidth]{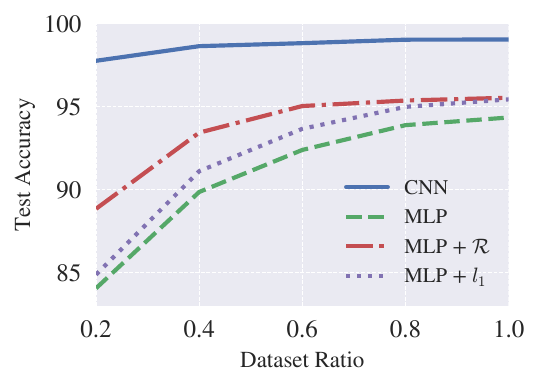}
    \caption{Test accuracy vs. dataset ratio on MNIST on torus for various models.}
    \label{fig:test_accuracy_vs_p}
\end{figure}

In \cref{fig:test_accuracy_vs_p}, we investigate the effect of dataset size on test accuracy. We observe that the introduction of weight-sharing regularization serves as a beneficial inductive bias for the network, effectively narrowing the generalization gap, particularly when using a reduced dataset ratio.

\begin{figure}[!h]
    \centering
    \includegraphics[width=0.38\linewidth]{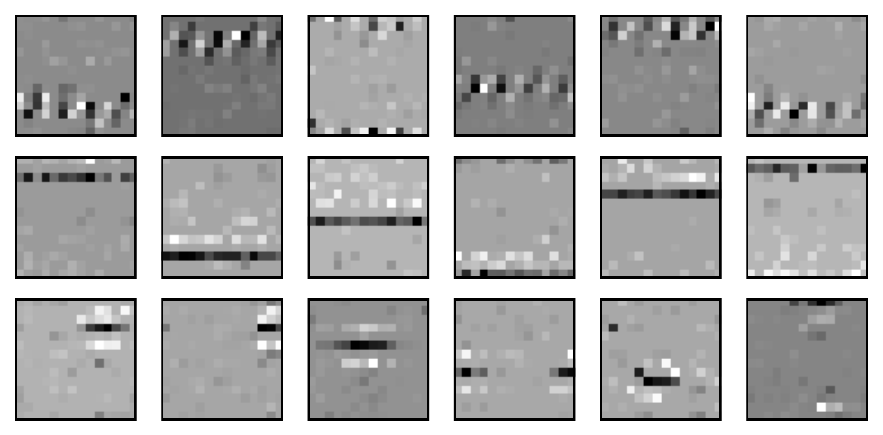}
    \caption{The emergence of learned convolution-like filters in a fully connected network with weight-sharing regularization in MNIST on torus.}
    \label{fig:mnist_weight}
\end{figure}

\begin{figure}[!h]
    \centering
    \includegraphics[width=0.4\linewidth]{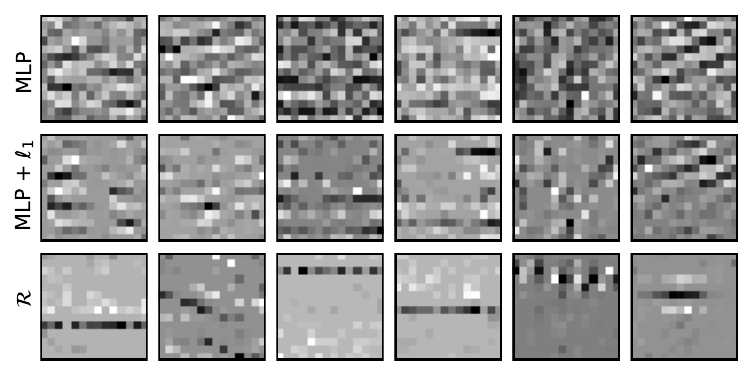}
    \caption{A random sample of learned filters from the 256 filters with the highest number of non-zero weights in MNIST on torus.}
    \label{fig:mnist_weight_random}
\end{figure}

We visualize the learned filters in \cref{fig:mnist_weight,fig:mnist_weight_random}.
We are able to find filters that look like translations of one another, similar to what one finds in a convolutional layer. These are illustrated in \cref{fig:mnist_weight}. We manually grouped these filters by examining the top 256 filters with the highest number of non-zero weights. A random selection of these filters for each method is shown in \cref{fig:mnist_weight_random}. The same process was carried out for the CIFAR10 task in \cref{fig:learned_conv_selected,fig:learned_conv_random}.

\subsection{CIFAR10}
This experiment is primarily designed based on the experiments in ~\citet{neyshabur2020towards}. We train a shallow convolutional neural network with one convolutional layer followed by two fully connected layers. The number of output channels of the convolutional layer is 64. The fully connected network corresponding to this shallow convolutional network has approximately 75M parameters. Batch normalization is applied after each of the first two layers. A learning-rate of 0.1 and no momentum is used for all experiments, along with a cosine annealed learning-rate schedule, similar to ~\citet{neyshabur2020towards}.

For the $\beta$-LASSO baseline we use their reported optimal $\beta=50$ and a corresponding $\rho$ of 0.98 for our own methods. We conducted a sweep for the regularization coefficient of $\beta$-LASSO (referred to as $\lambda$) which resulted in 0.00001 being selected.

For our methods, a hyperparameter search resulted in the following values: for MLP + $\mathcal{R}$, $\alpha = 0.001$; for MLP + $\ell_1$, $\beta = 0.001$; and for MLP + $\mathcal{R}$ + $\ell_1$, $\alpha = 0.0002$ and $\beta = 0.001$.

\subsection{Runtime Comparison}\label{app:run_time_compare}

\begin{table}[!h]
    \centering
    \caption{Approximate runtime of each algorithm per epoch on CIFAR10 experiments.}
    \label{table:runtime}
    \begin{tabular}{l c} 
    \toprule
    \textbf{Algorithm} & \textbf{Time (sec)}\\
    \midrule
    Sequential PAV & 1800\\
    Search Collisions & 200\\
    Imminent Collisions (v2) & 50\\
    No Regularization & 10\\
    \bottomrule
    \end{tabular}
    \vspace{0.5em}
\end{table}

\subsection{Worst-case Runtime Comparison}

We compare the actual runtime of our implementations of \ica and \sca for worst-case inputs on a NVIDIA V100 GPU. \cref{fig:worst-case} shows the results. As expected, \sca is exponentially faster when there are enough processors available.

\begin{figure}[!h]
    \centering
    \includegraphics[width=0.35\linewidth]{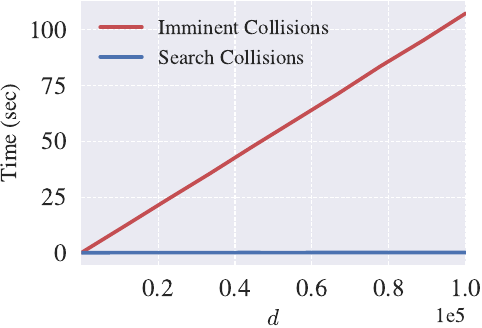}
    \includegraphics[width=0.35\linewidth]{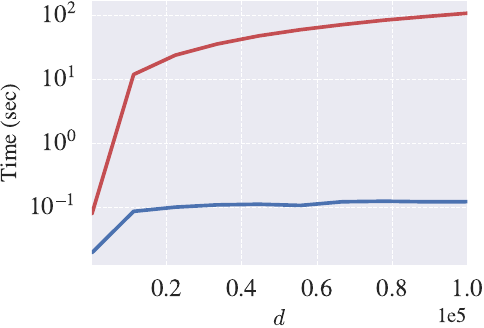}
    \caption{Worst-case runtimes of \sca vs. \ica. \textbf{(left)} linear scale \textbf{(right)} log scale.}
    \label{fig:worst-case}
\end{figure}

\end{document}